\documentclass{aamas2016}

\usepackage{amssymb}
\usepackage{amsmath}

\usepackage{amsthm}
\usepackage{bbm}
\usepackage{graphicx}
\usepackage{subfigure}
\usepackage{url}
\usepackage{threeparttable}
\usepackage{multirow}
\usepackage[table,xcdraw]{xcolor}
\usepackage[linesnumbered,vlined,ruled]{algorithm2e}
\usepackage{tabularx, booktabs}

\theoremstyle{plain}
\newtheorem{pro}{Property}
\newtheorem{thm}{Theorem}

\newtheorem{cor}[thm]{Corollary}

\begin{document}

\title{Optimal Target Assignment and Path Finding\\ for Teams of Agents}

\numberofauthors{2}

\author{
\alignauthor Hang Ma\\
  \affaddr{Department of Computer Science}\\ \affaddr{University of Southern California}\\
  \email{hangma@usc.edu}
\alignauthor  Sven Koenig\\
  \affaddr{Department of Computer Science}\\ \affaddr{University of Southern California}\\
  \email{skoenig@usc.edu}
}

\maketitle

\begin{abstract}
  We study the TAPF (combined target-assignment and path-finding) problem for
  teams of agents in known terrain, which generalizes both the anonymous and
  non-anonymous multi-agent path-finding problems. Each of the teams is given
  the same number of targets as there are agents in the team. Each agent has
  to move to exactly one target given to its team such that all targets are
  visited. The TAPF problem is to first assign agents to targets and then plan
  collision-free paths for the agents to their targets in a way such that the
  makespan is minimized.  We present the CBM (Conflict-Based Min-Cost-Flow)
  algorithm, a hierarchical algorithm that solves TAPF instances optimally by
  combining ideas from anonymous and non-anonymous multi-agent path-finding
  algorithms. On the low level, CBM uses a min-cost max-flow algorithm on a
  time-expanded network to assign all agents in a single team to targets and
  plan their paths. On the high level, CBM uses conflict-based search to
  resolve collisions among agents in different teams. Theoretically, we prove
  that CBM is correct, complete and optimal. Experimentally, we show the
  scalability of CBM to TAPF instances with dozens of teams and hundreds of
  agents and adapt it to a simulated warehouse system.
\end{abstract}


\category{I.2.8}{Artificial Intelligence}{Problem Solving, Control Methods, and Search}[graph and tree search strategies, heuristic methods]
\category{I.2.11}{Artificial Intelligence}{Distributed Artificial Intelligence}[intelligent agents, multi-agent systems]



\terms{Algorithms, Performance, Experimentation}

\ccsdesc[500]{Computing methodologies~Planning and scheduling}
\ccsdesc[500]{Computing methodologies~Multi-agent planning}
\ccsdesc[500]{Computing methodologies~Multi-agent systems}
\ccsdesc[500]{Computing methodologies~Intelligent agents}
\ccsdesc[300]{Computing methodologies~Planning for deterministic actions}
\ccsdesc[300]{Computing methodologies~Robotic planning}
\ccsdesc[300]{Computing methodologies~Discrete space search}
\ccsdesc[300]{Computing methodologies~Motion path planning}
\ccsdesc[300]{Mathematics of computing~Network flows}
\ccsdesc[300]{Theory of computation~Shortest paths}
\ccsdesc[300]{Computer systems organization~Robotic autonomy}


\keywords{heuristic search; Kiva (Amazon Robotics) systems; multi-agent path
  finding; multi-robot path finding; network flow; path planning; robotics;
  target assignment; team work; warehouse automation}

\section{Introduction}

Teams of agents often have to assign targets among themselves and then plan
collision-free paths to their targets. Examples include autonomous aircraft
towing vehicles~\cite{airporttug16}, automated warehouse systems~\cite{kiva},
office robots~\cite{DBLP:conf/ijcai/VelosoBCR15} and game characters in video
games \cite{WHCA}. For example, in the near future, autonomous aircraft towing
vehicles might tow aircraft all the way from the runways to their gates (and
vice versa), reducing pollution, energy consumption, congestion and human
workload. Today, autonomous warehouse robots already move inventory pods all
the way from their storage locations to the inventory stations that need the
products they store (and vice versa), see Figure \ref{kiva}.

\begin{figure}
  \centering
  \includegraphics[width=\columnwidth]{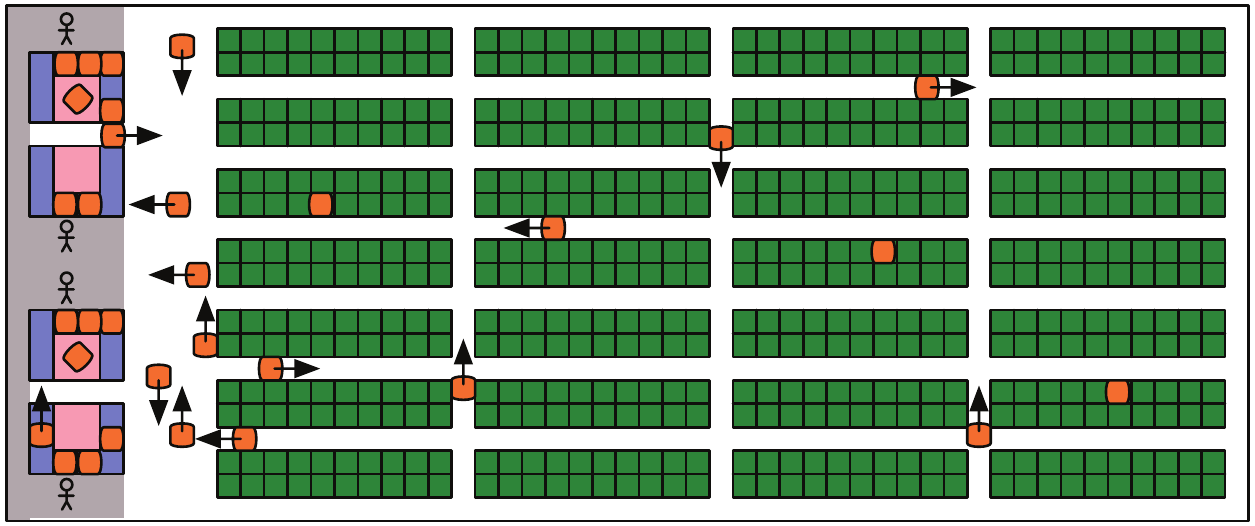}
  \caption{A typical Kiva warehouse system~\protect\cite{kiva}.}
  \label{kiva}
\end{figure}

We therefore study the {\bf TAPF (combined target-assignment and path-finding)
  problem} for teams of agents in known terrain. The agents are partitioned
into teams. Each team is given the same number of unique targets (goal
locations) as there are agents in the team. The TAPF problem is to assign
agents to targets and plan collision-free paths for the agents from their
current locations to their targets in a way such that each agent moves to
exactly one target given to its team, all targets are visited and the makespan
(the earliest time step when all agents have reached their targets and stop
moving) is minimized. Any agent in a team can be assigned to a target of the
team, and the agents in the same team are thus exchangeable. However, agents
in different teams are not exchangeable.

\subsection{Related Work}

The TAPF problem generalizes the anonymous and non-anonymous MAPF (multi-agent
path-finding) problems:

\begin{itemize}

\item The {\bf anonymous MAPF problem} (sometimes called goal-invariant MAPF
  problem) results from the TAPF problem if only one team exists (that
  consists of all agents). It is called ``anonymous'' because any agent can be
  assigned to a target, and the agents are thus exchangeable. The anonymous
  MAPF problem can be solved optimally in polynomial time
  \cite{YuLav13STAR}. Anonymous MAPF solvers use, for example, the
  polynomial-time {\bf max-flow algorithm} on a time-expanded network
  \cite{YuLav13STAR} (an idea that originated in the operations research
  literature~\cite{SurveyDinaymicFlow}) or graph-theoretic algorithms
  \cite{AAAI15-MacAlpine}.

\item The {\bf non-anonymous MAPF problem} (often just called MAPF problem)
  results from the TAPF problem if every team consists of exactly one agent
  and the number of teams thus equals the number of agents. It is
  called ``non-anonymous'' because only one agent can be assigned to a target
  (meaning that the assignments of agents to targets are pre-determined), and
  the agents are thus non-exchangeable.  The non-anonymous MAPF problem is
  NP-hard to solve optimally and even NP-hard to approximate within any
  constant factor less than 4/3 \cite{MaAAAI16}. Non-anonymous MAPF solvers
  use, for example, reductions to problems from satisfiability, integer linear
  programming or answer set programming
  ~\cite{YuLav13ICRA,erdem2013general,Surynek15} or optimal, bounded
  suboptimal or suboptimal search algorithms
  \cite{ODA,EPEJAIR,MStar,DBLP:journals/ai/SharonSGF13,ICBS,PushAndSwap,PushAndRotate,ECBS,DBLP:conf/socs/CohenUK15},
  such as the optimal {\bf CBS (conflict-based search) algorithm}
  \cite{DBLP:journals/ai/SharonSFS15}.

\end{itemize}

Research so far has concentrated on these two extreme cases. Yet, many
real-world applications fall between the extreme cases because the number of
teams is larger than one but smaller than the number of agents, which is why
we study the TAPF problem in this paper. The TAPF problem is NP-hard to solve
optimally and even NP-hard to approximate within any constant factor less than
4/3 if more than one team exists \cite{MaAAAI16}. It is unclear how to
generalize anonymous MAPF algorithms to solving the TAPF problem.
Straightforward ways of generalizing non-anonymous MAPF algorithms to solving
the TAPF problem have difficulties with either scalability (due to the
resulting large state spaces), such as searching over all assignments of
agents to targets to find optimal solutions, or solution quality, such as
assigning agents to targets with algorithms such as
\cite{Tovey2005,ZhengIJCAI} and then planning collision-free paths for the
agents with non-anonymous MAPF algorithms (perhaps followed by improving the
assignment and iterating \cite{WagnerSoCS12}) to find sub-optimal solutions.

\subsection{Contribution}

We present the {\bf CBM (Conflict-Based Min-Cost-Flow) algorithm} to bridge
the gap between the extreme cases of anonymous and non-anonymous MAPF
problems. CBM solves the TAPF problem optimally by simultaneously assigning
agents to targets and planning collision-free paths for them, while utilizing
the polynomial-time complexity of solving the anonymous MAPF problem for all
agents in a team to scale to a large number of agents. CBM is a hierarchical
algorithm that combines ideas from anonymous and non-anonymous MAPF
algorithms. It uses CBS on the high level and a min-cost max-flow algorithm
\cite{Successive} on a time-expanded network on the low level. Theoretically,
we prove that CBM is correct, complete and optimal. Experimentally, we show
the scalability of CBM to TAPF instances with dozens of teams and hundreds of
agents and adapt it to a simulated warehouse system.

\section{TAPF}

In this section, we formalize the TAPF problem and show how it can be solved
via a reduction to the integer multi-commodity flow problem on a
time-expanded network.

\subsection{Definition and Properties}

For a TAPF instance, we are given an undirected connected graph $G = (V,E)$
(whose vertices $V$ correspond to locations and whose edges $E$ correspond to
ways of moving between locations) and $K$ teams $team_1 \ldots team_K$. Each
{\em team} $team_i$ consists of $K_i$ agents $a^i_1 \ldots a^i_{K_i}$. Each
agent $a^i_j$ has a unique {\em start vertex} $s^i_j$. Each team $team_i$ is
given unique {\em targets} (goal vertices) $g^i_1 \ldots g^i_{K_i}$. Each
agent $a^i_j$ must move to a unique target $g^i_{j'}$. An {\em assignment} of
agents in team $team_i$ to targets is thus a one-to-one mapping $\varphi^i$,
determined by a permutation on $1\ldots K_i$, that maps each agent $a^i_j$ in
$team_i$ to a unique target $g^i_{j'} = \varphi^i(a^i_j)$ of the same team. A
{\em path} for agent $a^i_j$ is given by a function $l^i_j$ that maps each
integer time step $t = 0\ldots\infty$ to the vertex $l^i_j(t) \in V$ of the
agent in time step $t$. A {\em solution} consists of paths for all agents that
obey the following conditions:

\begin{enumerate}

\item $\forall i, j:~l^i_j(0) = s^i_j$ (each agent starts at its start vertex);

\item $\forall i, j~\exists \mbox{a minimal}~T^i_j~\forall t \geq
  T^i_j:~l^i_j(t) = \varphi^i(a^i_j)$ (each agent ends at its target);

\item $\forall i, j, t:~(l^i_j(t) = l^i_j(t+1)$ or $(l^i_j(t), l^i_j(t+1))\in
  E)$ (each agent always stays at its current vertex or moves to an adjacent
  vertex);

\item $\forall a^i_j, a^{i'}_{j'}, t~\mbox{with}~a^i_j \neq a^{i'}_{j'}:
  l^i_j(t)\neq l^{i'}_{j'}(t)$ (there are no vertex collisions since different
  agents never occupy the same vertex at the same time);

\item $\forall a^i_j, a^{i'}_{j'}, t~\mbox{with}~a^i_j \neq a^{i'}_{j'}:
  (l^i_j(t) \neq l^{i'}_{j'}(t+1)$ or $l^{i'}_{j'}(t) \neq l^i_j(t+1))$ (there
  are no edge collisions since different agents never move along the same edge
  in different directions at the same time).

\end{enumerate}

Given paths for all agents in team $team_i$, the {\em team cost} of team
$team_i$ is $\max_{j}T^i_j$ (the earliest time step when all agents in the
team have reached their targets and stop moving). Given paths for all agents,
the {\em makespan} is $\max_{i,j} T^i_j$ (the earliest time step when all
agents have reached their targets and stop moving). The task is to find an
{\em optimal solution}, namely one with minimal makespan. Note that a
(non-anonymous) MAPF instance can be obtained from a TAPF instance by fixing
the assignments of agents to targets. Any solution of a TAPF instance is thus
also a solution of a (non-anonymous) MAPF instance on the same graph for a
suitable assignment of agents to targets. Since the makespan of any optimal (non-anonymous) MAPF solution is bounded by $O(|V|^3)$
\cite{YuR14}, the makespan of any optimal TAPF solution is also
bounded by $O(|V|^3)$.

We define a collision between an agent agent $a^i_j$ in team $team_i$ and a
different agent $a^{i'}_{j'}$ in team $team_{i'}$ to be either a {\em vertex
  collision} ($team_i$, $team_{i'}$, $l$, $t$) [if $l = l^i_j(t) =
l^{i'}_{j'}(t)$ and thus both agents occupy the same vertex at the same time]
or an {\em edge collision} ($team_i$, $team_{i'}$, $l_1$, $l_2$, $t$) [if $l_1
= l^i_j(t) = l^{i'}_{j'}(t+1)$ and $l_2 = l^{i'}_{j'}(t) = l^i_j(t+1)$ and
thus both agents move along the same edge in different directions at the same
time].  Likewise, we define a constraint to be either a {\em vertex
  constraint} ($team_i$, $l$, $t$) [that prohibits any agent in $team_i$ from
occupying vertex $l$ in time step $t$] or an {\em edge constraint} ($team_i$,
$l_1$, $l_2$, $t$) [that prohibits any agent in team $team_i$ from moving from
vertex $l_1$ to vertex $l_2$ between time steps $t$ and $t+1$].

\subsection{Solution via Reduction to Flow Problem}
 \label{TAPF and Network Flow}

 \begin{figure}
  \centering
  \includegraphics[width=\columnwidth]{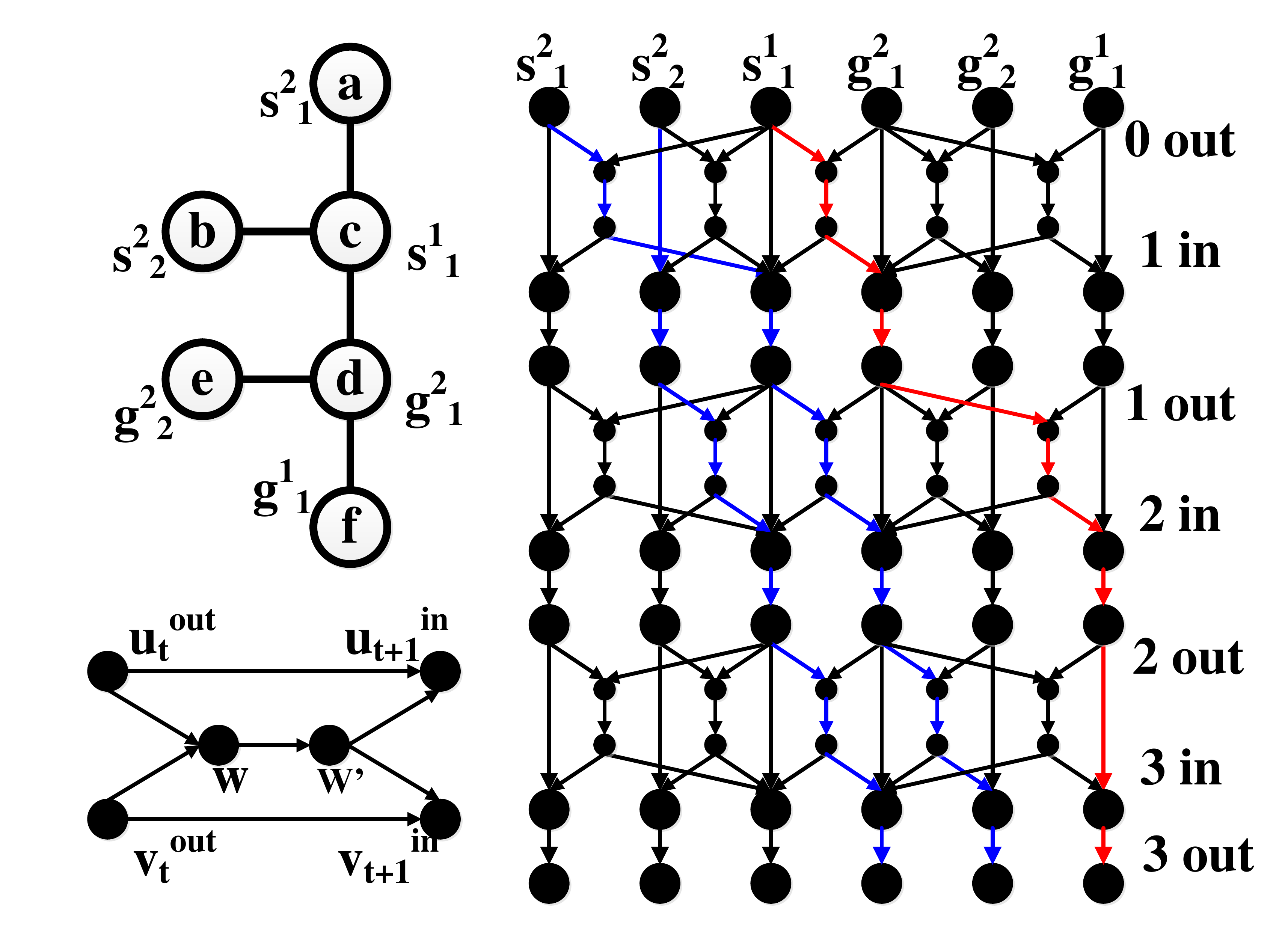}
  \caption{An example of constructing and finding a feasible integer
    multi-commodity flow for a 3-step time-expanded network. The TAPF instance
    consists of two teams. Team 1 consists of agent $\{a^1_1\}$ with target
    $\{g^1_1\}$. Team 2 consists of agents $\{a^2_1, a^2_2\}$ with targets
    $\{g^2_1, g^2_2\}$. The red edges represents a flow for commodity type 1,
    which corresponds to a path for the agent in Team 1. The blue edges
    represent a flow for commodity type 2, which corresponds to paths for the
    two agents in Team 2. The flow thus corresponds to an assignment of agent
    $a^1_1$ to target $g^1_1$, agent $a^2_1$ to target $g^2_2$ and agent
    $a^2_2$ to target $g^2_1$ as well as the (optimal) solution $\{\langle
    c,d,f,f\rangle\}$ and $\{\langle a,c,d,e\rangle, \langle
    b,b,c,d\rangle\}$.}\label{extend}
\end{figure}

Given a TAPF instance on undirected {\em graph} $G=(V,E)$ and a limit $T$ on
the number of time steps, we construct a $T$-step time-extended network using
a reduction that is similar to that from the (non-anonymous) MAPF problem to
the integer multi-commodity flow problem~\cite{YuLav13ICRA} (the idea of which
is an extension of \cite{YuLav13STAR}).  A {\em $T$-step time-extended
  network} is a directed network $\mathcal{N} = (\mathcal{V},\mathcal{E})$
with vertices $\mathcal{V}$ and directed edges $\mathcal{E}$ that have unit
capacity. Each vertex $v \in V$ is translated to a vertex $v_t^{out} \in
\mathcal{V}$ for all $t=0\ldots T$ (which represents vertex $v$ at the end of
time step $t$) and a vertex $v_t^{in} \in \mathcal{V}$ for all $t=1\ldots T$
(which represents vertex $v$ in the beginning of time step $t$). There is a
supply of one unit of commodity type $i$ at vertex ${(s^i_j)}_0^{out}$ and a
demand of one unit of commodity type $i$ at vertex ${(g^i_j)}^{out}_T$ for all
$i = 1\ldots K$ and $j = 1\ldots K_i$. Each vertex $v \in V$ is also
translated to an edge $(v_t^{out},v_{t+1}^{in}) \in \mathcal{E}$ for all
$t=0\ldots T-1$ (which represents an agent staying at vertex $v$ between time
steps $t$ and $t+1$).  Each vertex $v \in V$ is also translated to an edge
$(v_t^{in},v_t^{out}) \in \mathcal{E}$ for all $t=1\ldots T$ (which prevents
vertex collisions of the form $(*, *, v, t)$ among all agents since only one
agent can occupy vertex $v$ between time steps $t$ and $t+1$). Each edge
$(u,v) \in E$ is translated to a {\em gadget} of vertices in $\mathcal{V}$ and
edges in $\mathcal{E}$ for all $t=0\ldots {T-1}$, which consists of two
auxiliary vertices $w, w' \in \mathcal{V}$ that are unique to the gadget (but
have no subscripts here for ease of readability) and the edges $(u_t^{out},w),
(v_t^{out},w), (w,w'), (w',u_{t+1}^{in}), (w,v_{t+1}^{in}) \in
\mathcal{E}$. This gadget prevents edge collisions of the forms $(*, *, u, v,
t)$ and $(*, *, v, u, t)$ among all agents since only one agent can move along
the edge $(u,v)$ in any direction between time steps $t$ and $t+1$.
Figure~\ref{extend} shows a simple example. The following theorem holds by
construction and can be proved in a way similar to the one for the reduction
of the (non-anonymous) MAPF problem to the integer multi-commodity flow
problem \cite{YuLav13ICRA}:

\begin{thm}
\label{makespan and multiflow}
There is a correspondence between all feasible integer multi-commodity flows
on the $T$-step time-extended network of a number of unit that equals the
number of agents and all solutions of the TAPF instance with makespans of at
most $T$.
\end{thm}

An optimal solution can therefore be found by starting with $T=0$ and
iteratively checking for increasing values of $T$ whether a {\bf feasible
  integer multi-commodity flow} of a number of units that equals the number of
agents exists for the corresponding $T$-step time-expanded network (which is
an NP-hard problem), until an upper bound on $T$ is reached (such as the one
provided in~\cite{YuR14}). Each $T$-step time-expanded network is translated
in the standard way into an ILP (integer linear program), which is then solved
with an ILP algorithm. We evaluate this ILP-based TAPF solver experimentally
in Section~\ref{Comparisons}. The anonymous MAPF problem results from the TAPF
problem if only one team exists (that consists of all agents). The following
corollary thus follows from \cite{YuLav13STAR}:

\begin{cor}
  The TAPF problem can be solved optimally in polynomial time if only one team
  exists.
\end{cor}

\section{Conflict-Based Min-Cost Flow}

In this section, we present the {\bf CBM (Conflict-Based Min-Cost-Flow)
  algorithm}, a hierarchical algorithm that solves TAPF instances optimally.
On the high level, CBM considers each team to be a meta-agent. It uses CBS to
resolve collisions among meta-agents, that is, agents in different teams. CBS
is a form of best-first search on a tree, where each node contains a set of
constraints and paths for all agents that obey these constraints, move all
agents to unique targets of their teams and result in no collisions among
agents in the same team. On the low level, CBM uses a polynomial-time min-cost
max-flow algorithm \cite{Successive} on a time-expanded network to assign all
agents in a single team to unique targets of the same team and plan paths for
them that obey the constraints imposed by the currently considered high-level
node and result in no collisions among the agents in the team. Since the running
time of CBS on the high level can be exponential in the number of collisions
that need to be resolved \cite{DBLP:journals/ai/SharonSFS15}, CBM uses edge
weights on the low level to bias the search so as to reduce the possibility of
creating collisions with agents in different teams.

The idea of biasing the search on the low level has been used before for
solving the (non-anonymous) MAPF problem with CBS \cite{ECBS}. Similarly, the
idea of grouping some agents into a meta-agent on the high level and planning
paths for each group on the low level has been used before for solving the
(non-anonymous) MAPF problem with CBS \cite{DBLP:journals/ai/SharonSFS15} but
faces the difficulty of having to identify good groups of agents. The best way
to group agents can often be determined only experimentally and varies
significantly among MAPF instances. On the other hand, grouping all agents in
a team into a meta-agent for solving the TAPF problem is a natural way of
grouping agents since the assignments of agents in the same team to targets
and their paths strongly depend on each other and should therefore be planned
together on the low level. For example, if an agent is assigned to a different
target, then many of the agents in the same team typically need to be assigned
to different targets as well and have their paths re-planned. Also, the lower
level can then use a polynomial-time max-flow algorithm on a time-expanded
network to assign all agents in a single team to targets and find paths for
them due to the polynomial-time complexity of the corresponding anonymous MAPF
problem.

\subsection{High-Level Search of CBM}
\begin{algorithm}[t]
\small
\caption{High-Level Search of CBM}
\label{High-Level Search}
    $Root.constraints \gets \emptyset$\;
    $Root.paths \gets \emptyset$\;
    \For{\textnormal{\textbf{each} $team_i$}}
    {
      \If{\textnormal{Lowlevel($team_i$, $Root$) returns no paths}}
      {
        \Return No solution exists\;
      }
      Add the returned paths to $Root.paths$\;
    }
    $Root.key \gets$ Makespan($Root.paths$)\;
    $Priorityqueue \gets \{Root\}$\;
    \While{$Priorityqueue \neq \emptyset$}
    {
        $N \gets Priorityqueue.pop()$\;
        \If{\textnormal{Findcollisions($N.paths$) returns no collisions}}
        {
            \Return Solution is $N.paths$\;
        }
        $Collision \gets$ earliest collision found\;
        \For{\textnormal{\textbf{each} $team_i$ involved in $Collision$}}
        {
            $N'\gets$ Newnode() /* with parent $N$ */\;
            $N'.constraints \gets N.constraints$\;
            $N'.paths \gets N.paths$\;
            Add one new constraint for $team_i$ to $N'.constraints$\;
            \If{\textnormal{Lowlevel($team_i$, $N'$) returns paths}}
            {
                Update $N'.paths$ with the returned paths\;
                $N'.key \gets$ $max$(Makespan($N'.paths$), $N.key$)\;
                $Priorityqueue.insert(N')$\;
            }
        }
    }
    \Return No solution exists\;
\end{algorithm}

On the high level, CBM performs a best-first search on a binary tree, see
Algorithm~\ref{High-Level Search}. Each node $N$ contains constraints
$N.constraints$ and paths for all agents $N.paths$ that obey these
constraints, move all agents to unique targets of their teams and result in no
collisions among agents in the same team. All nodes are stored in a priority
queue. The priority queue initially consists of only the root node $Root$ with
no constraints and paths for all agents that move all agents to unique targets
of their teams, result in no collisions among agents in the same team and
minimize the team cost of each team [Lines 1-8]. If the priority queue is
empty, then CBM terminates unsuccessfully [Line 23]. Otherwise, CBM always
chooses a node $N$ in the priority queue with the smallest key [Line 10]. The
key of a node is the makespan of its paths. (Ties are broken in favor of the
node whose paths have the smallest number of colliding teams.) If the paths of
node $N$ have no colliding agents, then they are a solution and CBM terminates
successfully with these paths [Lines 11-12].  Otherwise, CBM determines all
collisions between two agents (which have to be in different teams) and then
resolves a collision $Collision$ whose time step $t$ is smallest [Line
13]. (We have evaluated different ways of prioritizing the collisions,
including the one suggested in \cite{ICBS}, but have not observed significant
differences in the resulting running times of CBM.) Let the two colliding
agents be in $team_i$ and $team_{i'}$. CBM then generates two child nodes
$N_1$ and $N_2$ of node $N$, both of which inherit the constraints and paths
from their parent node [Lines 15-17]. If the collision is a vertex collision
$(team_i, team_{i'}, l, t)$ or, equivalently, $(team_{i'}, team_i, l, t)$,
then CBM adds the vertex constraint $(team_i, l, t)$ to the constraints of
node $N_1$ and the vertex constraint $(team_{i'}, l, t)$ to the constraints of
node $N_2$ [Line 18]. If the collision is an edge collision $(team_i,
team_{i'}, l_1, l_2, t)$ or, equivalently, $(team_{i'}, team_i, l_2, l_1, t)$,
then CBM adds the edge constraint $(team_i, l_1, l_2, t)$ to the constraints
of node $N_1$ and the edge constraint $(team_{i'}, l_2, l_1, t)$ to the
constraints of node $N_2$ [Line 18]. For each of the two new nodes, say node
$N_1$, the low-level search is called to assign all agents in team $team_i$ to
unique targets of the same team and find paths for them that obey the
constraints of node $N_1$ and result in no collisions among the agents in the
team. If the low-level search successfully returns such paths, then CBM
updates the paths of node $N_1$ by replacing the paths of all agents in team
$team_i$ with the returned ones, updates the key of node $N_1$ and inserts it
into the priority queue [Lines 19-22]. Otherwise, it discards the node.

\subsection{Low-Level Search of CBM}

On the low level, Lowlevel($team_i$,$N$) assigns all agents in team $team_i$
to unique targets of the same team and finds paths for them that obey all
constraints of node $N$ (namely all vertex constraints of the form ($team_i$,
*, *) and all edge constraints of the form ($team_i$, *, *, *)) and result in
no collisions among the agents in the team.

Given a limit $T$ on the number of time steps, CBM constructs the $T$-step
time-expanded network from Section \ref{TAPF and Network Flow} with the
following changes: a) There is only a single commodity type $i$ since CBM
considers only the single team $team_i$. There is a supply of one unit of this
commodity type at vertex ${(s^i_j)}_0^{out}$ and a demand of one unit of this
commodity type at vertex ${(g^i_j)}^{out}_T$ for all $j = 1\ldots K_i$. b) To
obey the vertex constraints, CBM removes the edge $(l^{in}_t, l^{out}_t)$ from
$\mathcal{E}$ for each vertex constraint of the form $(team_i, l, t)$. c) To
obey the edge constraints, CBM removes the edges $((l_1)^{out}_t, w)$ and
$(w', (l_2)^{in}_{t+1})$ from $\mathcal{E}$ for all gadgets that correspond to
edge $(l_1, l_2) \in E$ for each edge constraint of the form $(team_i, l_1,
l_2, t)$. Let $\mathcal{V} = \mathcal{V'}$ be the set of (remaining) vertices
and $\mathcal{E'}$ be the set of remaining edges.

Similar to the procedure from Section \ref{TAPF and Network Flow}, CBM
iteratively checks for increasing values of $T$ whether a feasible integer
single-commodity flow of $K_i$ units exists for the corresponding $T$-step
time-expanded network, which can be done with the polynomial-time max-flow
algorithm that finds a {\bf feasible maximum flow}. CBM can start with $T$
being the key of the parent node of node $N$ since it is a lower bound on the
new key of node $N$ due to Line 21. (For $N = Root$, CBM starts with $T=0$.)
During the earliest iteration when the max-flow algorithm finds a feasible
flow of $K_i$ units, the call returns successfully with the paths for the
agents in the team that correspond to the flow. If $T$ reaches an upper bound
on the makespan of an optimal solution (such as the one provided
in~\cite{YuR14}) and no feasible flow of $K_i$ units was found, then the call
returns unsuccessfully with no paths.

CBM actually implements Lowlevel($team_i$,$N$) in a more sophisticated way to
avoid creating collisions between agents in team $team_i$ and agents in other
teams by adding edge weights to the $T$-step time-expanded network. CBM sets
the weights of all edges in $\mathcal{E'}$ to zero initially and then modifies
them as follows: a) To reduce vertex collisions, CBM increases the weight of
edge $(v^{in}_t, v^{out}_t) \in \mathcal{E'}$ by one for each vertex $v =
l^{i'}_{j'}(t) \in V$ in the paths of node $N$ with $i' \neq i$ to reduce the
possibility of an agent of team $team_i$ occupying the same vertex at the same
time step as an agent from a different team. b) To reduce edge collisions, CBM
increases the weight of edge $(v^{in}_t, w) \in \mathcal{E'}$ by one for each
edge $(u = l^{i'}_{j'}(t), v = l^{i'}_{j'}(t+1)) \in E$ in the paths of node
$N$ with $i' \neq i$ (where $w$ is the auxiliary vertex of the gadget that
corresponds to edge $(u,v)$ and time step $t$) to reduce the possibility of an
agent of team $team_i$ moving along the same edge in a different direction but
at the same time step as an agent from a different team.

CBM uses the procedure described above, except that it now uses a min-cost
max-flow algorithm (instead of a max-flow algorithm) that finds a {\bf flow of
  minimal weight among all feasible maximal flows}. In particular, it uses the
successive shortest path algorithm \cite{Successive}, a generalization of the
Ford-Fulkerson algorithm that uses Dijkstra's algorithm to find a path of
minimal weight for one unit of flow. The complexity of the successive shortest
path algorithm is $O(U(|\mathcal{E'}| + |\mathcal{V'}|\log|\mathcal{V'}|))$,
where $O(|\mathcal{E'}| + |\mathcal{V'}|\log|\mathcal{V'}|)$ is the complexity
of Dijkstra's algorithm and $U$ is the value of the feasible maximal flow,
which is bounded from above by $K_i$. The number of times that the successive
shortest path algorithm is executed is bounded from above by the chosen upper
bound on the makespan of an optimal solution, which in turn is bounded from
above by $O(|V|^3)$. Thus, each low-level search runs in polynomial time.

\subsection{Analysis of Properties}

We use the following properties to prove that CBM is correct, complete and
optimal.

\begin{pro} \label{p1} There is a correspondence between all feasible integer
  flows of $K_i$ units on the $T$-step time-extended network constructed for
  team $team_i$ and node $N$ and all paths for agents in team $team_i$ that a)
  obey the constraints of node $N$, b) move all agents from their start
  vertices to unique targets of their team, c) result in no collisions among
  agents in team $team_i$ and d) result in a team cost of team $team_i$ of at
  most $T$.
\end{pro}

\noindent {\em Reason.} The property holds by construction and can be proved
in a way similar to the one for the reduction of the (non-anonymous) MAPF
problem to the integer multi-commodity flow problem \cite{YuLav13ICRA}:

\noindent Left to right: Assume that a flow is given that has the stated
properties. Each unit flow from a source to a sink corresponds to a path
through the time-extended network from a unique source to a unique sink. Thus,
it can be converted to a path for an agent such that all such paths together
have the stated properties: Properties a and d hold by construction of the
time-extended network; Property b holds because a flow of $K_i$ units uses all
supplies and sinks; and Property c holds since the flows neither share
vertices nor edges.

\noindent Right to left: Assume that paths are given that have the stated
properties. If necessary, we extend the paths by letting the agents stay at
their targets. Each path now corresponds to a path through the time-extended
network (due to Properties a and d) from a unique source to a unique sink (due
to Property b) that does not share directed edges with the other such paths
(due to Property c). Thus, it can be converted to a unit flow such that all
such unit flows together respect the unit capacity constraints and form a flow
of $K_i$ units.

\begin{pro} \label{p2}
CBM generates only finitely many nodes.
\end{pro}

\noindent {\em Reason.} The constraint added on Line 18 to a child node is
different from the constraints of its parent node since the paths of its
parent node do not obey it. Overall, CBM creates a binary tree of finite depth
since only finitely many different vertex and edge constraints exist and thus
generates only finitely many nodes.

\begin{pro} \label{p3}
  Whenever CBM inserts a node into the priority queue, its key is finite.
\end{pro}

\noindent {\em Reason (by induction).} The property holds for the root
node. Assume that it holds for the parent node of some child node. The key of
the child node is the maximum of the key of the parent node and the team costs
of all teams for the paths of the child node. The key of the parent node is
finite due to the induction assumption. The low level returned the paths for
each team successfully at some point in time and all team costs are thus
finite as well.

\begin{pro} \label{p4}
  Whenever CBM chooses a node on Line 10 and the paths of the node have no
  colliding agents, then CBM correctly terminates with a solution with finite
  makespan of at most the value of its key.
\end{pro}

\noindent {\em Reason.} The key of the node is finite according to Property
\ref{p3}, and the makespan of its paths is at most the value of its key due to
Line 21.

\begin{pro} \label{p5}
  CBM chooses nodes on Line 10 in non-decreasing order of their keys.
\end{pro}

\noindent {\em Reason.} CBM performs a best-first search, and the key of a
parent node is most the key of any of its child nodes due to Line 21.

\begin{pro} \label{p6}
  The smallest makespan of any solution that obeys the constraints of a parent
  node is at most the smallest makespan of any solution that obeys the
  constraints of any of its child nodes.
\end{pro}

\noindent {\em Reason.} The solutions that obey the constraints of a parent
node are a superset of the solutions that obey the constraints of any of its
child nodes since the constraints of the parent node are a subset of the
constraints of any of its child nodes.

\begin{pro} \label{p7}
  The key of a node is at most the makespan of any solution that obeys its
  constraints.
\end{pro}

\noindent {\em Reason (by induction).} The property holds for the root
node. Assume that it holds for the parent node $N$ of any child node $N'$ and
that the paths for team $team_i$ were updated in the child node. Let $x$ be
the smallest makespan of any solution that obeys the constraints of the parent
node and $y$ be the smallest makespan of any solution that obeys the
constraints of the child node. We show in the following that the key of the
parent node and the team costs of all teams for the paths of the child node
are all at most $y$. Then, the key of the child node is also at most $y$ since
it is the maximum of all these quantities, and the property holds. First,
consider the key of the parent node. The key of the parent node is at most $x$
due to the induction assumption, which in turn is at most $y$ due to Property
\ref{p6}.  Second, consider any team different from team $team_i$. Then, the
team cost of the team for the paths of the child node is equal to the team
cost of the team for the paths of the parent node (since the paths were not
updated in the child node and are thus identical), which in turn is at most
the key of the parent node (since the key of the parent node is the maximum of
several quantities that include the team cost of the team for the paths of the
parent node), which in turn is at most $y$ (as shown directly above). Finally,
consider team $team_i$. When the low level finds new paths for team $team_i$,
it starts with $T$ being the key of the parent node, which is at most $y$ (as
shown directly above). Thus, the max-cost min-flow algorithm on a $T$-step
time-expanded network constructed for team $team_i$ and the child node finds a
feasible integer flow of $K_i$ units for $T \leq y$ since there exists a
solution with makespan $y$ that obeys the constraints of the child node. The
team cost of the corresponding paths for team $team_i$ is at most $T$ due to
Property \ref{p1}.

\begin{thm}
CBM is correct, complete and optimal.
\end{thm}

\begin{proof}
  Assume that no solution to a TAPF instance exists and CBM does not terminate
  unsuccessfully on Line 5. Then, whenever CBM chooses a node on Line 10, the
  paths of the node have colliding agents (because otherwise a solution would
  exist due to Property \ref{p4}). Thus, the priority queue eventually becomes
  empty and CBM terminates unsuccessfully on Line 23 since it generates only
  finitely many nodes due to Property \ref{p2}.

  Now assume that a solution exists and the makespan of an optimal solution is
  $x$. Assume, for a proof by contradiction, that CBM does not terminate with
  a solution with makespan $x$. Thus, whenever CBM chooses a node on Line 10
  with a key of at most $x$, the paths of the node have colliding agents
  (because otherwise CBM would correctly terminate with a solution with
  makespan at most $x$ due to Property \ref{p4}).  A node whose constraints
  the optimal solution obeys has a key of at most $x$ due to Property 7. The
  root note is such a node since the optimal solution trivially obeys the
  (empty) constraints of the root node. Whenever CBM chooses such a node on
  Line 10, the paths of the node have colliding agents (as shown directly
  above since its key is at most $x$). CBM thus generates the child nodes of
  this parent node, the constraints of at least one of which the optimal
  solution obeys and which CBM thus inserts into the priority queue with a key
  of at most $x$. Since CBM chooses nodes on Line 10 in non-decreasing order
  of their keys due to Property \ref{p5}. it chooses infinitely many nodes on
  Line 10 with keys of at most $x$, which is a contradiction with Property
  \ref{p2}.
\end{proof}

\section{Experiments}

In this section, we describe the results of four experiments on a 2.50 GHz
Intel Core i5-2450M PC with 6 GB RAM. First, we compare CBM to four other TAPF
or MAPF solvers. Second, we study how CBM scales with the number of agents in
each team. Third, we study how CBM scales with the number of agents.  Fourth,
we apply CBM to a simulated warehouse system.

\subsection{Experiment 1: Alternative Solvers}
\label{Comparisons}

\begin{table*}
\tiny
\centering
\caption{Results on 30$\times$30 4-neighbor grids with randomly blocked cells for different numbers of agents.}
\label{30X30}
\resizebox{\textwidth}{!}{
\begin{tabular}{|c|c|c|c|c|c|c|c|c|c|c|c|c|c|c|c|}
\hline
     & \multicolumn{3}{c|}{CBM (TAPF)}  & \multicolumn{3}{c|}{Unweighted CBM (TAPF)}                                                        & \multicolumn{3}{c|}{ILP (TAPF)}                                                                   & \multicolumn{3}{c|}{CBS (MAPF)}                                                           & \multicolumn{3}{c|}{ILP (MAPF)}                                                             \\ \hline
agts & mkspn & time & success & mkspn                     & time                        & success                     & mkspn                     & time                          & success                     & mkspn                     & time                        & success                     & mkspn                     & time                          & success                     \\ \hline
10   & 22.34    & 0.34 & 1       & {\color[HTML]{FE0000} 22.08} & {\color[HTML]{FE0000} 0.41} & {\color[HTML]{FE0000} 0.72} & 22.34                        & 18.24                         & 1                           & 36.36                        & 0.03                        & 1                           & 36.36                        & 8.66                          & 1                           \\
15   & 23.88    & 0.57 & 1       & {\color[HTML]{FE0000} 24.64} & {\color[HTML]{FE0000} 1.06} & {\color[HTML]{FE0000} 0.44} & 23.88                        & 35.44                         & 1                           & 37.32                        & 0.05                        & 1                           & 37.32                        & 15.31                         & 1                           \\
20   & 25.06    & 0.78 & 1       & {\color[HTML]{FE0000} 23.73} & {\color[HTML]{FE0000} 2.06} & {\color[HTML]{FE0000} 0.22} & {\color[HTML]{FE0000} 24.74} & {\color[HTML]{FE0000} 62.85}  & {\color[HTML]{FE0000} 0.94} & 39.84                        & 0.55                        & 1                           & 39.84                        & 30.30                         & 1                           \\
25   & 25.20    & 1.07 & 1       & {\color[HTML]{FE0000} 22.25} & {\color[HTML]{FE0000} 1.58} & {\color[HTML]{FE0000} 0.08} & {\color[HTML]{FE0000} 24.76} & {\color[HTML]{FE0000} 88.55}  & {\color[HTML]{FE0000} 0.82} & 40.44                        & 0.12                        & 1                           & 40.44                        & 43.76                         & 1                           \\
30   & 26.26    & 1.71 & 1       & {\color[HTML]{FE0000} 31}    & {\color[HTML]{FE0000} 6.73} & {\color[HTML]{FE0000} 0.02} & {\color[HTML]{FE0000} 24.70} & {\color[HTML]{FE0000} 108.75} & {\color[HTML]{FE0000} 0.66} & 41.92                        & 0.21                        & 1                           & 41.92                        & 65.86                         & 1                           \\
35   & 26.50    & 1.92 & 1       & {\color[HTML]{FE0000} -}     & {\color[HTML]{FE0000} -}    & {\color[HTML]{FE0000} 0}    & {\color[HTML]{FE0000} 24.65} & {\color[HTML]{FE0000} 121.99} & {\color[HTML]{FE0000} 0.46} & 42.50                        & 1.55                        & 1                           & 42.50                        & 81.83                         & 1                           \\
40   & 27.60    & 2.95 & 1       & {\color[HTML]{FE0000} -}     & {\color[HTML]{FE0000} -}    & {\color[HTML]{FE0000} 0}    & {\color[HTML]{FE0000} 25.29} & {\color[HTML]{FE0000} 152.98} & {\color[HTML]{FE0000} 0.14} & {\color[HTML]{FE0000} 43.69} & {\color[HTML]{FE0000} 4.82} & {\color[HTML]{FE0000} 0.98} & {\color[HTML]{FE0000} 43.53} & {\color[HTML]{FE0000} 115.53} & {\color[HTML]{FE0000} 0.98} \\
45   & 27.20    & 3.66 & 1       & {\color[HTML]{FE0000} -}     & {\color[HTML]{FE0000} -}    & {\color[HTML]{FE0000} 0}    & {\color[HTML]{FE0000} 24.29} & {\color[HTML]{FE0000} 161.52} & {\color[HTML]{FE0000} 0.14} & {\color[HTML]{FE0000} 42.41} & {\color[HTML]{FE0000} 2.60} & {\color[HTML]{FE0000} 0.92} & {\color[HTML]{FE0000} 42.37} & {\color[HTML]{FE0000} 133.47} & {\color[HTML]{FE0000} 0.98} \\
50   & 27.90    & 5.32 & 1       & {\color[HTML]{FE0000} -}     & {\color[HTML]{FE0000} -}    & {\color[HTML]{FE0000} 0}    & {\color[HTML]{FE0000} 24.50} & {\color[HTML]{FE0000} 161.95} & {\color[HTML]{FE0000} 0.04} & {\color[HTML]{FE0000} 43.96} & {\color[HTML]{FE0000} 7.95} & {\color[HTML]{FE0000} 0.96} & {\color[HTML]{FE0000} 42.86} & {\color[HTML]{FE0000} 166.99} & {\color[HTML]{FE0000} 0.86} \\ \hline
\end{tabular}
}
\end{table*}

We compare our optimal TAPF solver CBM to two optimal (non-anonymous) MAPF
solvers, namely a) the CBS solver provided by the authors of
\cite{DBLP:journals/ai/SharonSFS15} and b) the ILP-based MAPF solver provided
by the authors of~\cite{YuLav13ICRA}, and two optimal TAPF solvers, namely a)
an unweighted version of CBM that runs the polynomial-time max-flow algorithm
on a time-expanded network without edge weights (instead of the min-cost
max-flow algorithm on a time-expanded network with edge weights) on the low
level and b) an ILP-based TAPF solver (based on the ILP-based MAPF solver)
that casts a TAPF instance as a series of integer multi-commodity flow
problems as described in Section \ref{TAPF and Network Flow}, each of which it
models as an ILP and solves with the ILP solver Gurobi 6.0 (www.gurobi.com).

For Experiment 1, each team consists of 5 agents but the number of agents
varies from 10 to 50, resulting in $2\ldots10$ teams. For each number of
agents, we generate 50 TAPF instances from the same 50 $30 \times 30$
4-neighbor grids with $10\%$ randomly blocked cells by randomly assigning
unique start cells to agents and unique targets to teams. For the MAPF
solvers, we convert each TAPF instance to a (non-anonymous) MAPF instance by
randomly assigning the agents in each team to unique targets of the same team.

Table~\ref{30X30} shows the success rates as well as the means of the
makespans and running times (in seconds) over the instances that are solved
within a time limit of 5 minutes each. Red entries indicate that some
instances are not solved within the time limit, while dashed entries indicate
that all instances are not solved within the time limit. CBM solves all TAPF
instances within the time limit.

\subsubsection{CBS and the ILP-Based MAPF Solver}

Both MAPF solvers solve most of the MAPF instances within the time limit. The
running times of CBM and CBS are similar because, on the low level, both the
min-cost max-flow algorithm of CBM (for a single team) and the A* algorithm of
CBS (for a single agent) are fast. Optimal solutions of the TAPF instances
have smaller makespans than optimal solutions of the MAPF instances due to the
freedom of assigning agents to targets for the TAPF instances rather than
assigning them randomly for the MAPF instances.

\subsubsection{Unweighted CBM}
\label{Comparison to Unweighted CBM}

Unweighted CBM solves less than half of all TAPF instances within the time
limit if the number of agents is larger than 10 due to the large number of
collisions among agents in different teams produced by the max-flow algorithm
on the low level in tight spaces with many agents, which results in a large
number of node expansions by CBS on the high level. We conclude that biasing
the search on the low level is important for CBM to solve all TAPF instances
within the time limit.

\subsubsection{ILP-Based TAPF Solver}

The ILP-based TAPF solver solves less than half of all TAPF instances within
the time limit if the number of agents is larger than 30, and its running time
is much larger than that of CBM. The success rates and running times of the
the ILP-based TAPF solver tend to be larger than those of the ILP-based MAPF
solver even though the ILP formulation of a TAPF instance has fewer variables
than that of the corresponding MAPF instance (since the number of commodity
types equals the number of teams for the TAPF instance but the number of
agents for the MAPF instance). However, the variables in the ILP formulation
of the MAPF instance are Boolean variables while those in the ILP formulation
of the TAPF instance are integer variables. Furthermore, the ILP-based MAPF solver uses the maximum over all agents of the length of a shortest path of each agent as the starting value of $T$ for the time-expanded network while the ILP-based TAPF solver solves the LP formulation of the max-flow problem that finds paths for each team (ignoring other teams) and then uses the maximum over all teams of the team costs of the paths as the starting value $T$ for
the time-expanded network.

\subsection{Experiment 2: Team Size}

\begin{figure}
\subfigure {\raisebox{25pt}{
\tiny
  \begin{tabular}[b]{|>{\hspace{-5pt}}c<{\hspace{-5pt}}|>{\hspace{-5pt}}c<{\hspace{-5pt}}|>{\hspace{-5pt}}c<{\hspace{-5pt}}|}
\hline
$K$    & mkspn & time    \\ \hline
2      & 11.1     & 2.75 \\
4      & 15.9     & 4.30 \\
5      & 17.12    & 2.56  \\
10     & 23.04    & 5.53 \\
20     & 29.32    & 6.06 \\
25     & 30.88    & 6.44  \\
50     & 39.76    & 12.15 \\ \hline
\end{tabular}
}
}\hspace{-0.5cm}\hfill
\subfigure {
  \includegraphics[width=0.78\columnwidth]{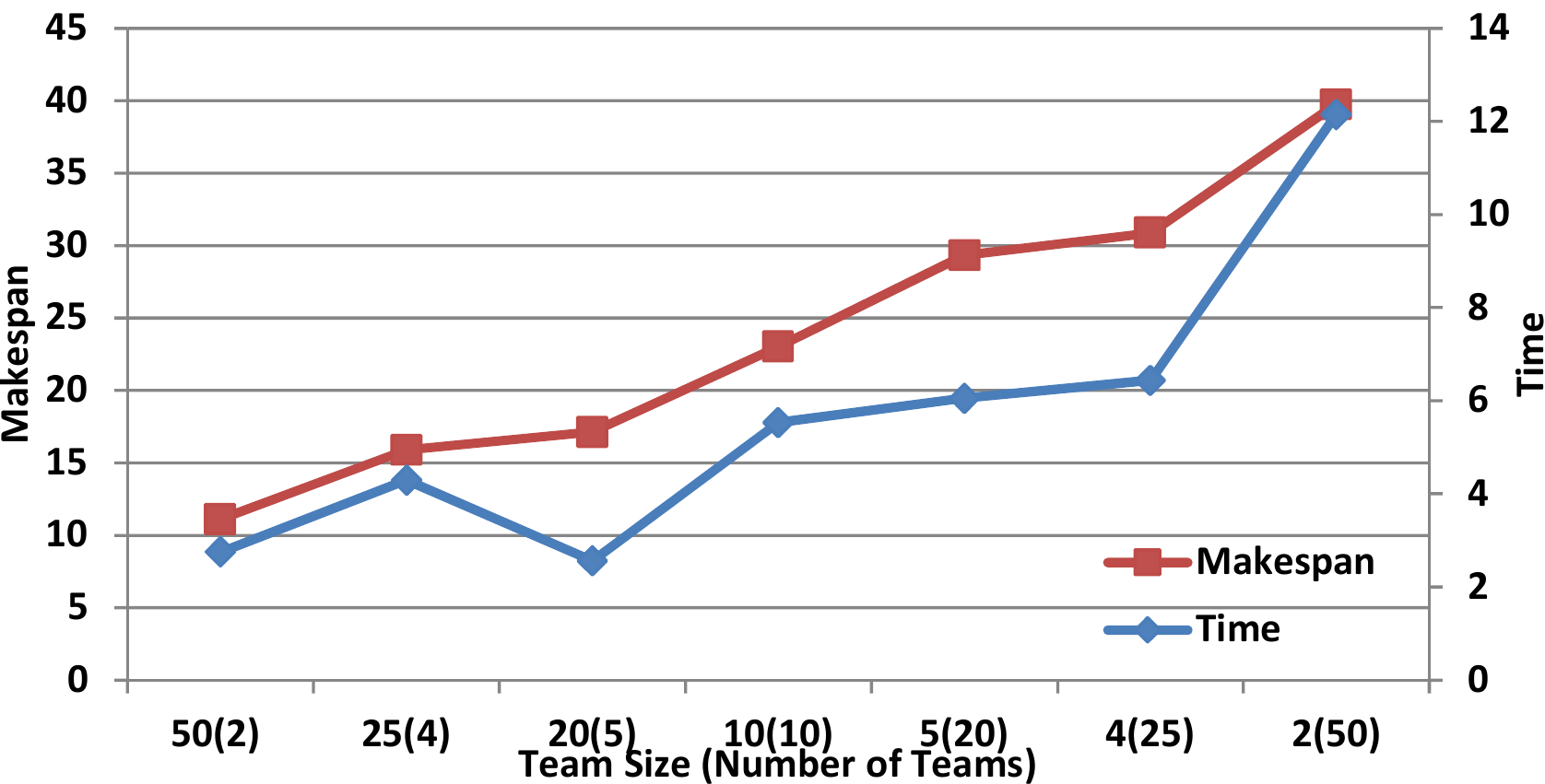}
}  \caption{Results on 30$\times$30 4-neighbor grids with randomly blocked
  cells for different team sizes. ($K$ is the number of teams.)}\label{100agents}
\end{figure}

For Experiment 2, there are 100 agents but the number of agents in a team
(team size) varies from 50 to 2, resulting in $K=2\ldots50$ teams. For each
team size, we generate 50 TAPF instances as described before.

Table~\ref{100agents} shows the means of the makespans and running times (in
seconds) over the instances that are solved within a time limit of 5 minutes
each. CBM solves all TAPF instances within the time limit. For large team
sizes and thus small numbers of teams, the makespans are small because CBM has
more freedom to assign agents to targets. The running times are also small
because the min-cost max-flow algorithm on the low level is fast even for
large numbers of agents while CBS on the high level is fast because it needs
to resolve collisions among agents in different teams but there are only a
small number of teams. Thus, it is advantageous for teams to consist of as
many agents as possible.

\subsection{Experiment 3: Number of Agents}

\begin{figure}
\subfigure {\raisebox{15pt}{
\tiny
\begin{tabular}[b]{|>{\hspace{-5.5pt}}c<{\hspace{-5.5pt}}|>{\hspace{-5.5pt}}c<{\hspace{-5.5pt}}|>{\hspace{-5.5pt}}c<{\hspace{-5.5pt}}|>{\hspace{-5.5pt}}c<{\hspace{-5.5pt}}|}
\hline
agts & mkspn & time   & success \\ \hline
100    & 30.10    & 6.14   & 1       \\
150    & 29.67    & 8.10   & 0.96    \\
200    & 32.09    & 12.97  & 0.94    \\
250    & 31.05    & 15.56  & 0.86    \\
300    & 32.09    & 25.42  & 0.7     \\
350    & 33.03    & 32.59  & 0.6     \\
400    & 34.19    & 59.69  & 0.42    \\
450    & 35.80    & 101.47 & 0.1     \\ \hline
\end{tabular}
}
}\hspace{-0.5cm}\hfill
\subfigure {
  \includegraphics[width=0.67\columnwidth]{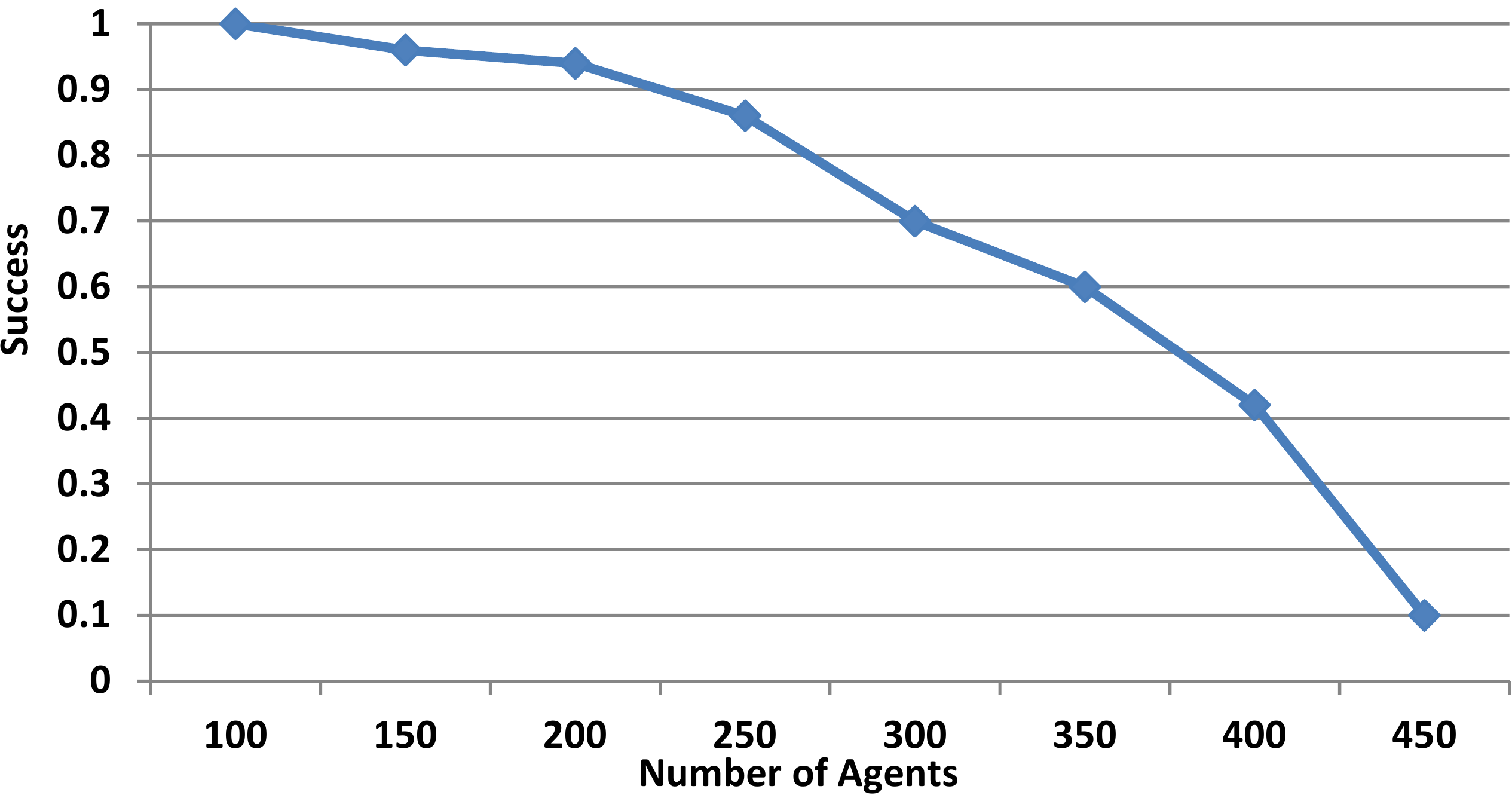}
}\caption{Results on 30$\times$30 4-neighbor grids with randomly blocked cells
  for different numbers of agents.}\label{5perteam}
\end{figure}

For Experiment 3, each team consists of 5 agents but the number of agents
varies from 100 to 450, resulting in $20\ldots90$ teams. For each number of
agents, we generate 50 TAPF instances as described before.

Table~\ref{100agents} shows the success rates as well as the means of the
makespans and running times (in seconds) over the instances that are solved
within a time limit of 5 minutes each. For 250 agents or fewer, the success
rate is larger than 85\%. Current (non-anonymous) MAPF algorithm are not able
to handle instances of this scale. As the number of agents increases, the
success rates decrease and the makespans and running times increase due to the
increasing number of collisions among agents in different teams produced by
the min-cost max-flow algorithm on the low level. For 450 agents, for example,
more than half of the unblocked cells are occupied by agents and thus many
start cells of agents are also targets for other agents.

\subsection{Experiment 4: Warehouse System}

\begin{figure}
  \centering
  \includegraphics[width=\columnwidth]{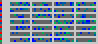}
  \caption{A randomly generated Kiva instance. CBM returns a solution with
    makespan 65 in 25.97 seconds. The length of the shortest path for both
    green drive units near the top-right corner to the green inventory station
    in the bottom-left corner is 64. Thus, at least one of them has to wait
    for at least one time step to enter the green inventory station and the
    solution found by CBM is optimal.}
  \label{random_kiva}
\end{figure}

We now apply CBM to a simulated {\bf Kiva (now: Amazon Robotics) warehouse
  system} \cite{kiva}. Figure~\ref{kiva} shows a typical grid layout with {\em
  inventory stations} on the left side and {\em storage locations} in the {\em
  storage area} to the right of the inventory stations. Each inventory station
has an entrance (purple cells) and an exit (pink cells). Each storage location
(green cell) can store one {\em inventory pod}. Each inventory pod consists of
a stack of trays, each of which holds bins with products. The autonomous
warehouse robots are called {\em drive units}. Each drive unit is capable of
picking up, carrying and putting down one inventory pod at a time.  As a team,
the drive units need to move inventory pods all the way from their storage
locations to the inventory stations that need the products they store (to ship
them to customers) and then back to the same or different empty storage
locations. After a drive unit enters an inventory station, the requested
product is removed from its inventory pod by a worker. Once drive units have
delivered all requested products for one shipment to the same inventory
station, the worker prepares the shipment to the customer.

Figure~\ref{random_kiva} shows a randomly generated Kiva instance. The light
grey cells are free space. The dark grey cells are storage locations occupied
by inventory pods and thus blocked. There are 7 inventory stations on the left
side. The red cells are their exits, and the other 7 cells with graduated
blue-green colors are their entrances. Drive units can enter and leave the
inventory stations one at a time through their entrances and exits,
respectively. The cells with graduated blue-green colors in the storage area
are occupied by drive units. Each drive unit needs to carry the inventory pod
in its current cell to the inventory station of the same color.

For Experiment 4, we generate 50 TAPF instances. Each instance has 420 drive
units. 210 ``incoming'' drive units start at randomly determined storage
locations: 30 drive units each need to move their inventory pods to the 7
inventory stations. In order to create difficult Kiva instances, we generate
the start cells of these drive units randomly among all storage locations
rather than cluster them according to their target inventory stations. 210
``outgoing'' drive units start at the inventory stations: 30 drive units each
need to move their inventory pods from the 7 inventory stations to the storage
locations vacated by the incoming drive units. The task is to assign the 210
outgoing drive units to the vacated storage locations and plan collision-free
paths for all 420 drive units in a way such that the makespan is minimized.
The incoming drive units that have the same inventory station as target are a
team (since they can arrive at the inventory station in any order), and all
outgoing drive units are a team.

So far, we have assumed that, for any TAPF instance, all start vertices are
unique, all targets are unique and each of the teams is given the same number
of targets as there are agents in the team but these assumptions are not
necessarily satisfied here.  1) The outgoing drive units that start at the
same inventory station all start at its exit. In this case, we change the
construction of the $T$-step time-extended network for the team of outgoing
drive units so that there is a supply of one unit at vertex $v_t^{out} \in
\mathcal{V'}$ for all $t = 0 \ldots 29$ and all vertices $v \in V$ that
correspond to exits of inventory stations. This construction forces the
outgoing drive units that start at the same inventory station to leave it one
after the other during the first 30 time steps. No further changes are
necessary. 2) The incoming drive units that have the same inventory station as
target all end at its entrance. In this case, we change the construction of
the $T$-step time-extended network for each team of incoming drive units so
that there is an auxiliary vertex with a demand of 30 units and vertex
$v_t^{out} \in \mathcal{V'}$ for all $t = 0 \ldots T$ is connected to the
auxiliary vertex with an edge with unit capacity and zero edge weight, where
$v \in V$ corresponds to the entrance of the inventory station. This
construction forces the incoming drive units to enter the inventory station at
different time steps. No further changes are necessary.  3) There could be
more empty storage locations than outgoing drive units. In this case, no
changes are necessary.

CBM finds solutions for 40 of the 50 Kiva instances within a time limit of 5
minutes each, yielding a success rate of 80\%. The mean of the makespan over
the solved Kiva instances is 63.73, and the mean of the running time is 91.61
seconds. Since early Kiva warehouse systems typically had about 200 drive
units in more spacious (and thus less challenging) warehouses and even
bounded-suboptimal (non-anonymous) MAPF algorithms that were specifically
designed for simulated Kiva warehouse systems do not scale well to hundreds of
agents \cite{DBLP:conf/socs/CohenUK15}, we conclude that CBM is a promising
TAPF algorithm for applications of real-world scale.

\section{Conclusions}

In this paper, we studied the TAPF (combined target-assignment and
path-finding) problem for teams of agents in known terrain to bridge the gap
between the extreme cases of anonymous and non-anonymous MAPF problems, as
required by many applications. We presented CBM, a hierarchical algorithm that
is correct, complete and optimal for solving the TAPF problem. CBM outperforms
(non-anonymous) MAPF algorithms in terms of both scalability and solution
quality in our experiments. It also generalizes to applications with dozens of
teams and hundreds of agents, which demonstrates its promise.

\section{Acknowledgments}

We thank Jingjin Yu for making the code of their ILP-based MAPF solver and
Guni Sharon for making the code of their CBS solver available to us. Our
research was supported by NASA via Stinger Ghaffarian Technologies as well as
NSF under grant numbers 1409987 and 1319966 and a MURI under grant number
N00014-09-1-1031. The views and conclusions contained in this document are
those of the authors and should not be interpreted as representing the
official policies, either expressed or implied, of the sponsoring
organizations, agencies or the U.S. government.

\bibliographystyle{abbrv}
\bibliography{references}

\end{document}